\documentclass[12pt, preprint]{alt2022} 

\title[Scale-Free Adversarial MAB]{Scale-Free Adversarial Multi Armed Bandits}
\usepackage{times}

\usepackage{hyperref}       
\usepackage{url}            
\usepackage{booktabs}       
\usepackage{amsfonts}       
\usepackage{nicefrac}       
\usepackage{microtype}      
\usepackage{xcolor}         

\usepackage{graphicx}
\usepackage{wrapfig}

\usepackage{thmtools}
\usepackage{thm-restate}
\usepackage{comment}

\newcommand{\breg}{\text{Breg}}

\altauthor{\Name{Sudeep Raja Putta} \Email{sp3794@columbia.edu}\\
\addr Columbia University
\AND
 \Name{Shipra Agrawal} \Email{sa3305@columbia.edu}\\
 \addr Columbia University}

\begin{document}

\maketitle

\begin{abstract}%
  We consider the  Scale-Free Adversarial Multi Armed Bandits(MAB) problem. At the beginning of the game, the player only knows the number of arms $n$.  It does not know the scale and magnitude of the losses chosen by the adversary or the number of rounds $T$. In each round, it sees bandit feedback about the loss vectors $l_1,\dots, l_T \in \mathbb{R}^n$. The goal is to bound its regret as a function of $n$ and norms of $l_1,\dots, l_T$. We design a bandit Follow The Regularized Leader (FTRL) algorithm, that uses an adaptive learning rate and give two different regret bounds, based on the exploration parameter used. With non-adaptive exploration, our algorithm has a regret of $\tilde{\mathcal{O}}(\sqrt{nL_2} + L_\infty\sqrt{nT})$ and with adaptive exploration, it has a regret of $\tilde{\mathcal{O}}(\sqrt{nL_2} + L_\infty\sqrt{nL_1})$. Here $L_\infty = \sup_t \| l_t\|_\infty$, $L_2 = \sum_{t=1}^T \|l_t\|_2^2$, $L_1 = \sum_{t=1}^T \|l_t\|_1$ and the $\tilde{\mathcal{O}}$ notation suppress logarithmic factors. These are the first MAB bounds that adapt to the $\|\cdot\|_2$, $\|\cdot\|_1$ norms of the losses. The second bound is the first data-dependent scale-free MAB bound as $T$ does not directly appear in the regret. We also develop a new technique for obtaining a rich class of local-norm lower-bounds for Bregman Divergences. This technique plays a crucial role in our analysis for controlling the regret when using importance weighted estimators of unbounded losses. This technique could be of independent interest.
\end{abstract}


\section{Introduction}

The Adversarial Multi Armed Bandit(MAB) problem proceeds as a sequential game of $T$ rounds between a player and an adversary. In each round $t=1,\dots, T$, the player selects a distribution $p_t$ over the $n$-arms and the adversary selects a loss vector $l_t$ belonging to some set $\mathcal{L} \subseteq \mathbb{R}^n$. An action $i_t$ is sampled from $p_t$ and the player observes the loss $l_t(i_t)$. The (expected) regret of the player is:

$$R_T = \mathbb{E}\left[\sum_{t=1}^T l_t(i_t) - \min_{i \in [n]} \sum_{t=1}^T l_t(i)\right]$$

We assume that the adversary is oblivious, i.e., the loss vectors $l_1,\dots, l_T$ are chosen before the game begins. So, the above expectation is with respect to the randomness in the player's strategy. The goal of the player is to sequentially select the distributions $p_1,\dots,p_T$ such that $R_T$ is minimized. The adversarial MAB problem has been studied extensively; we refer the reader to the texts of  \cite{DBLP:journals/ftml/BubeckC12,lattimore_szepesvari_2020,DBLP:journals/ftml/Slivkins19} for further details. Assuming that $\mathcal{L}$ is bounded, and the $\|\cdot\|_\infty$-Lipschitz constant $G$ is known to the player in advance (i.e. $\sup_{l \in \mathcal{L}} \|l\|_\infty = G<\infty$), the minimax rate of regret is known to be $\Theta(G\sqrt{nT})$.  The Exp3 algorithm \citep{DBLP:journals/siamcomp/AuerCFS02} has a $\mathcal{O}(G\sqrt{nT\log(n)})$ regret bound whereas the Poly-INF algorithm \citep{DBLP:conf/colt/AudibertB09} removes the $\sqrt{\log(n)}$ factor, achieving the optimal $\mathcal{O}(G\sqrt{nT})$ regret bound.  Exp3 and Poly-INF use $G$ in tuning the learning rate, which helps them achieve a linear dependence on $G$.

In this paper, we address the case when the player has no knowledge of $\mathcal{L}$. We consider \textit{Scale-Free} bounds for MABs, which aim to bound the regret in terms of $n$ and norms of the loss vectors $l_1,\dots,l_T$ for any sequence of loss vectors chosen arbitrarily by adversary. Scale-free bounds have been studied in the \textit{full-information} setting (where the player sees the complete vector $l_t$ in each round). For the Experts problem, which is the full-information counterpart of adversarial MAB, the AdaHedge algorithm \citep{ DBLP:journals/jmlr/RooijEGK14} has a scale-free regret bound of $\mathcal{O}(\sqrt{\log(n)(\sum_{t=1}^T \|l_t\|_\infty^2)})$. For the same problem, the Hedge algorithm \citep{DBLP:journals/jcss/FreundS97} has a regret bound of $\mathcal{O}(G\sqrt{T\log(n)})$ with knowledge of $G$. The scale-free bound is more general as it holds for any $l_1,\dots,l_T \in \mathbb{R}^n$, whereas the bound achieved by the Hedge algorithm  only holds provided that $\sup_t \|l_t\|_\infty < G$ where $G$ needs to be known in advance. 

\subsection{Our Contributions}
We present an algorithm for the scale-free MAB problem. By appropriately 
setting the parameters of this algorithm, we can achieve a scale-free regret upper-bound of either $\tilde{\mathcal{O}}(\sqrt{nL_2} + L_\infty\sqrt{nT})$, or $\tilde{\mathcal{O}}(\sqrt{nL_2} + L_\infty\sqrt{nL_1})$. Here $L_\infty = \sup_t \| l_t\|_\infty$, $L_2 = \sum_{t=1}^T \|l_t\|_2^2$, $L_1 = \sum_{t=1}^T \|l_t\|_1$ and the $\tilde{\mathcal{O}}$ notation suppress logarithmic factors. Our algorithm is also \textit{any-time} as it does not need to know the number of rounds $T$ in advance. Assuming $\sup_t \|l_t\|_\infty < G$, our first regret bound achieves linear dependence on $G$ (sans the hidden logarithmic terms). This bound is only $\tilde{\mathcal{O}}(\sqrt{n})$ factor larger than Poly-INF's regret of $\mathcal{O}(G\sqrt{nT})$. The second bound is the first completely data-dependent scale-free regret bound for MABs as it has no direct dependence on $T$. Moreover, these are the first MAB bounds that adapt to the $\|\cdot\|_2$, $\|\cdot\|_1$ norms of the losses. The only previously known scale-free result for MABs was $\mathcal{O}(L_\infty \sqrt{nT\log(n)})$ by \citet{hadiji2020adaptation}, which adapts to the $\| \cdot \|_\infty$ norm and is not completely data-dependent due to the $T$ in their bound.


In the analysis, we present a novel and
general technique to obtain \textit{local-norm} lower-bounds for \textit{Bregman divergences} induced by a special class of functions that are commonly used in online learning. These local-norm lower-bounds can be used to obtain regret inequalities as shown in \citet[Corollary 28.8]{lattimore_szepesvari_2020}. We use our technique to obtain a full-information regret inequality that holds for any arbitrary sequence of losses and is particularly useful in the bandit setting due to its local-norm structure. This technique could be of independent interest.

\subsection{Related Work}
\textbf{Scale-Free Regret.} As mentioned earlier, Scale-Free regret bounds were studied in the full information setting. The AdaHedge algorithm from \citet{DBLP:journals/jmlr/RooijEGK14} gives a scale-free bound for the experts problem. The AdaFTRL algorithm from \cite{DBLP:journals/tcs/OrabonaP18} extends these bounds to the general online convex optimization problem. We rely on the analysis of AdaFTRL as presented in \cite{koolen_2016}. For the MAB problem,  \citet{hadiji2020adaptation} show a scale-free bound of $\mathcal{O}(L_\infty \sqrt{nT\log(n)})$, which is close to the $\mathcal{O}(G\sqrt{nT\log(n)})$ bound of Exp3. Our scale-free bounds are more versatile as they are able to adapt to additional structure in the loss sequence, such as the case of sparse losses with large magnitude, i.e., when $L_2<< L_\infty^2 nT$ and $L_1 << L_\infty nT$. Even in the worst-case, our bounds are a factor of $\tilde{\mathcal{O}}(\sqrt{n})$ and  $\tilde{\mathcal{O}}(\sqrt{nL_\infty})$ larger than their bound respectivley.
\\

\noindent \textbf{Data-dependent Regret.} These bounds use a ``measure of hardness" of the sequence of loss vectors instead of $T$. Algorithms that have a data-dependent regret bound perform better than the worst-case regret, when the sequence of losses is ``easy" according to the measure of hardness used. For instance, First-order bounds \citep{DBLP:conf/alt/AllenbergAGO06, DBLP:conf/nips/FosterLLST16, DBLP:conf/uai/PogodinL19}, also known as small-loss or $L^\star$ bounds depend on $L^\star = \min_{i \in [n]}\sum_{t=1}^T l_t(i)$. Bounds that depend on the empirical variance of the losses were shown in \cite{DBLP:journals/jmlr/HazanK11, DBLP:conf/alt/BubeckCL18}. Path length bounds that depend on $\sum_{t=1}^{T-1}\|l_t-l_{t+1}\|$ or a similar quantity appear in \cite{DBLP:conf/colt/WeiL18, DBLP:conf/colt/BubeckLLW19}.  \citet{DBLP:journals/jmlr/ZimmertS21} give an algorithm that adapts to any stochastictiy present in the losses. Our bound is comparable to a result in \cite{DBLP:conf/alt/BubeckCL18}, where they derive a regret bound depending on $\sum_{t=1}^T \|l_t\|_2^2$. However, all these results assume either $\mathcal{L} = [0,1]^n$ or $\mathcal{L} = [-1,1]^n$. 
\\

\noindent \textbf{Effective Range Regret.} The effective range of the loss sequence is defined as $\sup_{t,i,j}| l_t(i) -  l_t(j)|$.  \citet{DBLP:conf/nips/GerchinovitzL16} showed that it is impossible to adapt to the effective range in adversarial MAB. This result does not contradict the existence of scale-free bounds as the effective range could be much smaller than, for instance, the complete range $\sup_{t,s,i,j}| l_t(i) -  l_s(j)|$. In fact, \citet{hadiji2020adaptation} already show a regret bound that adapts to the complete range. We do note that under some mild additional assumptions, \citet{DBLP:conf/alt/Cesa-BianchiS18} show that it is possible to adapt to the effective range.

\subsection{Organization}
In Section \ref{sec:algorithm} we present the scale-free MAB algorithm (Algorithm \ref{alg:SF_MAB}) and its scale-free regret bound (Theorem \ref{thm:main}). Section \ref{sec:preliminaries} introduces Potential functions, based on which we build our analysis. Section \ref{subsec:BregmanLowerBound} shows a technique for obtaining local-norm lower-bounds for Bregman divergences. Section \ref{subsec:FTRL} briefly discusses full-information FTRL, AdaFTRL and in Theorem \ref{thm:LogBarrierRegret} we obtain a regret inequality for AdaFTRL with the log-barrier regularizer. Theorem \ref{thm:main} is proved in Section \ref{sec:proof}.

\subsection{Notation}
Let $\Delta_n$ be the probability simplex $\{p\in \mathbb{R}^n: \sum_{i=1}^n p(i) = 1, p(i)\geq 0, i\in [n]\}$.  Let $\textbf{1}^{i}$ be the vector with $\textbf{1}^{i}(i)=1$ and $\textbf{1}^{i}(j)=0$ for all $j\neq i$. For $\epsilon \in (0,1]$, let $\textbf{1}^i_\epsilon = (1-\epsilon) \textbf{1}^i + \epsilon/n$. The all ones and all zeros vector are denoted by $\textbf{1}$ and $\textbf{0}$ respectively. Let $H_t$ be the history from time-step $1$ to $t$, i.e., $H_t = \{l_1(i_1), l_2(i_2),\dots,l_t(i_t)\}$.

\section{Algorithm}
\label{sec:algorithm}

 Consider for a moment, full-information strategies on $\Delta_n$. In the full information setting, in each round $t$, the player picks a point $p_t \in \Delta_n$. Simultaneously, the adversary picks a loss vector $l_t \in \mathbb{R}^n$. The player incurs a loss of $l_t^\top p_t$ and (unlike the bandit setting) {\it sees the entire vector $l_t$.} A full-information strategy $\mathcal{F}$ takes as input a sequence of loss vectors $l_1,\dots,l_t$ and outputs the next iterate $p_{t+1} \in \Delta_n$. A MAB strategy $\mathcal{B}$ can be constructed from a full-information strategy $\mathcal{F}$ along with two other components as follows:
\begin{enumerate}
    \item A sampling scheme $\mathcal{S}$, which constructs a sampling distribution $p'_t$ from the current iterate $p_t$. An arm $i_t$ is then sampled from $p'_t$ and the loss $l_t(i_t)$ is revealed to the player.
    \item An estimation scheme $\mathcal{E}$, that constructs an estimate $\tilde{l}_t$ of the loss vector $l_t$ using $l_t(i_t)$ and $p_t$.
    \item A full-information strategy $\mathcal{F}$, which computes the next iterate $p_{t+1}$ using all the estimates $\tilde{l}_1,\dots,\tilde{l}_t$.
\end{enumerate}
In fact, most existing MAB strategies in the literature can be described in the above framework with different choices of ${\cal S}, {\cal E}, {\cal F}$.

A delicate balance needs to be struck between $\mathcal{S}, \mathcal{E}$ and $\mathcal{F}$ in order to achieve a good regret bound for $\mathcal{B}$. Suppose the best arm in hindsight is $i_\star = \arg\min_{i \in [n]} \sum_{t=1}^T l_t(i) $ The expected regret of MAB strategy $\mathcal{B}$ can be decomposed as follows:
\begin{align*}
     &\mathbb{E}\left[\sum_{t=1}^T (l_t(i_t) - l_t(i^\star))\right] = \mathbb{E}\left[\sum_{t=1}^T l_t^\top(p'_t - \textbf{1}^{i^\star})\right] =\mathbb{E}\left[\sum_{t=1}^T l_t^\top(p'_t - p_t)\right] + \mathbb{E}\left[\sum_{t=1}^T l_t^\top(p_t - \textbf{1}^{i^\star})\right] \\
     &= \underbrace{\mathbb{E}\left[\sum_{t=1}^T l_t^\top(p'_t - p_t)\right]}_{(1)} + \underbrace{\mathbb{E}\left[\sum_{t=1}^T (l_t-\tilde{l}_t^\top)(p_t - \textbf{1}^{i^\star})\right]}_{(2)}+
     \underbrace{\mathbb{E}\left[\sum_{t=1}^T \tilde{l}_t^\top(p_t - \textbf{1}^{i^\star})\right]}_{(3)}
\end{align*}
Term (1) is due to the sampling scheme $\mathcal{S}$, term (2) is the effect of the estimation scheme $\mathcal{E}$ and term (3) is the expected regret of the full-information strategy $\mathcal{F}$ on the loss sequence $\tilde{l}_1,\dots,\tilde{l}_T$ compared to playing the fixed strategy $\textbf{1}^{i^\star}$.
\\

\noindent \textbf{Sampling Scheme.} A commonly used sampling scheme mixes $p_t$ with the uniform distribution using a parameter $\gamma$, i.e., $p'_t = (1-\gamma)p_t + \gamma/n$. Such schemes were first introduced in the seminal work of \citet{DBLP:journals/siamcomp/AuerCFS02} and have remained a mainstay in MAB algorithm design. We use a time-varying $\gamma$, i.e., we pick $p'_t = (1-\gamma_{t-1})p_t + \gamma_{t-1}/n$. Here $\gamma_{t-1}$ could be any measurable function of $H_{t-1}$.
\\

\noindent \textbf{Estimation Scheme.} We use the \textit{Importance Weighted}(IW) estimator which was also introduced by \citet{DBLP:journals/siamcomp/AuerCFS02}. It computes $\tilde{l}_t$ as:
$$\tilde{l}_t = \frac{l_t(i_t)}{p'_t(i_t)} \textbf{1}^{i_t}$$
Since the sampling distribution is $p'_t$, the IW estimator is an unbiased estimate of $l_t$:
$$\mathbb{E}_{i_t \sim p'_t}[\tilde{l}_t] = \sum_{i_t=1}^n  p'_t(i_t)\frac{l_t(i_t)}{p'_t(i_t)} \textbf{1}^{i_t} = l_t$$
Note that $p_t$ is a measurable function of $H_{t-1}$. Using the tower rule and the fact that $\mathbb{E}_{i_t \sim p'_t}[\tilde{l}_t] = l_t$, we can see that term (2) is $0$.
\\

\noindent \textbf{Full-information startegy.} For $\mathcal{F}$, there is a large variety of full-information algorithms that one could pick from. Most if not all of them belong to one of the two principle families of algorithms: Follow The Regularized Leader(FTRL) or Online Mirror Descent(OMD). Further, one also has to choose a suitable \textit{regularizer} $F$ within these algorithms for the particular application at hand. We refer to \cite{DBLP:books/daglib/0016248, DBLP:journals/ftml/Shalev-Shwartz12, DBLP:journals/ftopt/Hazan16, DBLP:journals/corr/abs-1912-13213, DBLP:conf/alt/JoulaniGS17, DBLP:journals/tcs/JoulaniGS20} for a detailed history and comparison of these algorithms. The particular algorithm we use is FTRL with a $H_t$ measurable, adaptive learning rate $\eta_t$ that resembles the adaptive schemes in AdaHedge \citep{DBLP:journals/jmlr/RooijEGK14} and AdaFTRL \citep{DBLP:journals/tcs/OrabonaP18}. 

The regret of $\mathcal{F}$ has an component called the \text{stability} term $\Psi_{p}:\mathbb{R}^n \to \mathbb{R}$. In the bandit case, $\mathcal{F}$ receives the IW estimates $\tilde{l}_t$. So, it is important that the stability term be bounded with IW estimates. Without going into any technical details, we note that it is desirable to have a stability term bounded by $\Psi_p(l)  \leq p^\top l^2$ as its expectation with IW estimates can be bounded.

Previous techniques to bound the stability term by $p^\top l^2$ relied on the assumptions on $l$, such as either $l\geq \textbf{0}$ or $l\geq -\textbf{1}$ (See \cite[Page 5]{DBLP:journals/corr/abs-1907-05772}). For arbitrary $l \in \mathbb{R}^n$, we show that it is possible to bound the stability term by $p^\top l^2$ using the \textit{log-barrier} regularizer. The procedure we develop to obtain this bound is the main technical contribution of our paper.

 The complete algorithm for the scale-free MAB problem is described below. We give two choices for the exploration parameter $\gamma_t$. A simple non-adaptive scheme that is similar to the one in \cite{hadiji2020adaptation}, where $\gamma_t \propto \frac{1}{\sqrt{t}}$ and an adaptive scheme that picks $\gamma_t$ in a fashion that resembles the adaptive learning rate scheme $\eta_t$.

\begin{algorithm2e}
\caption{Scale-Free Multi Armed Bandit}
\label{alg:SF_MAB}
\DontPrintSemicolon
Starting Parameters: $\eta_0=n,\gamma_0=1/2$\;
Regularizer $\displaystyle F(q) = \sum_{i=1}^n (f(q(i)) - f(1/n))$, where $f(x) = -\log(x)$\;
First iterate $p_1 = (1/n,\dots,1/n)$\;
\For{$t = 1$ to $T$}{
Sampling Scheme: $\displaystyle  p'_t = (1-\gamma_{t-1})p_t + \frac{\gamma_{t-1}}{n}$\;
Sample Arm $i_t \sim p'_t$ and see loss $l_t(i_t)$.\;
Estimation Scheme: $\displaystyle  \tilde{l}_t = \frac{l_t(i_t)}{p'_t(i_t)} \textbf{1}^{i_t}$\;
Compute $\gamma_t$ for next step: \\(Option 1) Non-adaptive $\gamma_t = \min(1/2,\sqrt{n/t})$ \\(Option 2) Adaptive $\displaystyle \gamma_t = \frac{n}{2n + \sum_{s=1}^t \Gamma_s(\gamma_{s-1})}$ where $\displaystyle \Gamma_t(\gamma) = \frac{\gamma |l_t(i_t)|}{(1-\gamma) p_t(i_t) + \gamma/n}$\;
Compute  $\displaystyle  \eta_t = \frac{n}{1+\sum_{s=1}^t M_s(\eta_{s-1})} $ where $\displaystyle  M_t(\eta) = \sup_{q \in \Delta_n} \left[ \tilde{l}_t^\top (p_t-q) - \frac{1}{\eta} \breg_{F}(q\|p_t) \right]$\;
Find next iterate using FTRL:
$\displaystyle  p_{t+1} = \arg \min_{q \in \Delta_n} \left[ F(q) + \eta_{t} \sum_{s=1}^{t}  q^\top \tilde{l}_s \right]$
}
\end{algorithm2e}
Our main result is the following regret bound for Algorithm \ref{alg:SF_MAB}.

\begin{restatable}{theorem}{main}
\label{thm:main}
For any $l_1,\dots,l_T \in \mathbb{R}^n$, the expected regret of Algorithm \ref{alg:SF_MAB} is at most:
\begin{enumerate}
    \item  $\tilde{\mathcal{O}}(\sqrt{nL_2} + L_\infty\sqrt{nT})$ if $\gamma_t$ is non-adaptive (Option 1) and $T\geq 4n$
    \item  $\tilde{\mathcal{O}}(\sqrt{nL_2} + L_\infty\sqrt{nL_1})$ if $\gamma_t$ is adaptive (Option 2)
\end{enumerate}
Where $L_\infty = \max_t \| l_t\|_\infty$, $L_2 = \sum_{t=1}^T \|l_t\|_2^2$, $L_1 = \sum_{t=1}^T \|l_t\|_1$.
\end{restatable}

\section{Preliminaries}
We begin by recalling a few definitions.
\label{sec:preliminaries}

\begin{definition}[Legendre function]
A continuous function $F:\mathcal{D} \to \mathbb{R}$ is Legendre if $F$ is strictly convex, continuously differentiable  on $\text{Interior}(\mathcal{D})$ and  $\lim_{x \to \mathcal{D}/\text{Interior}(\mathcal{D})} \|\nabla F(x)\| = +\infty$.
\end{definition}
For instance, the function $x\log(x)-x$, $-\sqrt{x}$, $-\log(x)$ are all Legendre on $(0,\infty)$

\begin{definition}[Bregman Divergence]
The Bregman Divergence of function $F$ is:

$$\breg_F(x\|y) = F(x)-F(y) - \nabla F(y)^\top (x-y).$$
\end{definition}

\begin{definition}[Potential Function] A function $\psi: (-\infty,a) \to (0,+\infty)$ for some $a \in \mathbb{R} \cup \{+\infty\}$ is called a Potential if it is convex, strictly increasing, continuously differentiable and satisfies:
    $$\lim_{x \to -\infty} \psi(x) = 0 \quad\text{ and } \quad \lim_{x \to a} \psi(x) = +\infty$$
\end{definition}

\begin{figure}[h]
        \centering
        \includegraphics[width=0.5\textwidth]{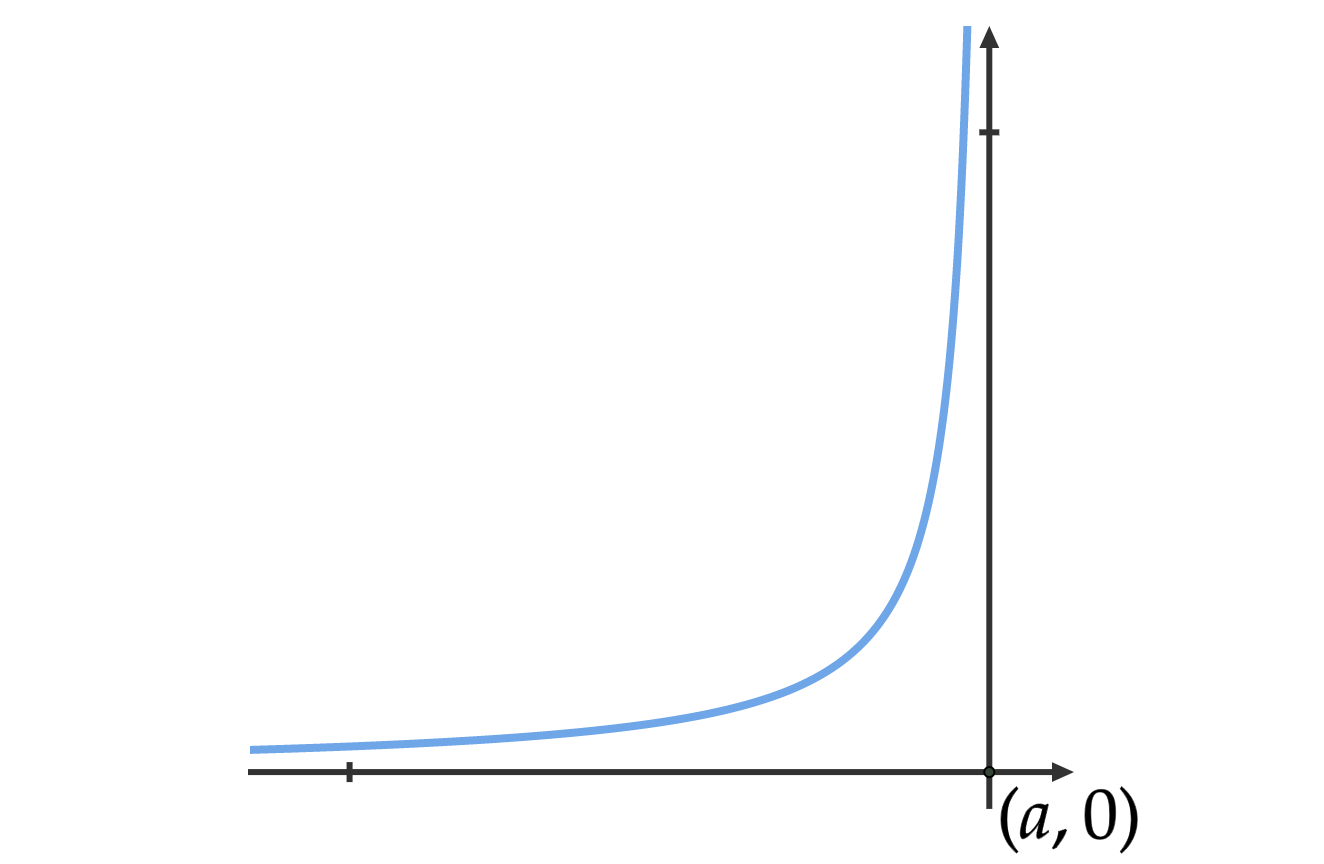}
        \caption{Potential Function}
        \label{fig:potential}
\end{figure}

For instance, $\exp(x)$ is a potential with $a=\infty$ and $-1/x$ is a potential with $a=0$. A potential function typically looks like Figure \ref{fig:potential}. Potentials were introduced in \citet{DBLP:conf/colt/AudibertB09, DBLP:journals/jmlr/AudibertBL11, DBLP:journals/mor/AudibertBL14} for analyzing the Implicitly Normalized Forecaster(INF) algorithm, of which Poly-INF is a specific case. 

Associated with a potential $\psi$, we define a function $f_\psi$ as the indefinite integral $f_\psi(z) = \int \psi^{-1}(z) dz + C$. Since the domain of $\psi^{-1}$ is $(0,\infty)$, the domain of $f_\psi$ is also $(0,\infty)$. For instance, if $\psi(x) = -1/x$ on the domain $(-\infty,0)$, the associated function is $f_\psi(x) = -\log(x) + C$.

Observe that $f_\psi'(z) = \psi^{-1}(z)$ and $f_\psi''(z) = \left[ \psi'(\psi^{-1}(z))\right]^{-1}$. Since $\psi$ is strictly convex and increasing, $\psi'>0$ and thus $f''_\psi >0$, making $f_\psi$ strictly convex. Moreover, $\lim_{z \to 0}\mid f_\psi'(z)\mid = \lim_{z \to 0}\mid \psi^{-1}(z)\mid = +\infty$. Thus $f_\psi$ is a Legendre function on $(0,\infty)$. Define the function $F_\psi:\mathbb{R}^n \to \mathbb{R}$ as $F_\psi(x) = \sum_{i=1}^n [f_\psi(x(i)) - f_\psi(1/n)]$. This function is Legendre on $(0,\infty)^n$.

Given a potential $\psi:(-\infty,a) \to (0,+\infty)$ and its associated function $f_\psi$, the Legendre-Fenchel dual of $f_\psi$ is $f_\psi^\star:(-\infty,a) \to  \mathbb{R}$ defined as $f_\psi^\star(u) = \sup_{z > 0}(zu - f_\psi(z))$. The supremum is achieved at $z={f'_\psi}^{-1}(u)=\psi(u)$. So we have that $f_\psi^\star(u) = u \psi(u) - f_\psi(\psi(u))$. This implies $ {f_\psi^\star}'(u) = \psi(u) $ and ${f_\psi^\star}''(u) = \psi'(u)$. Further, using integration by parts on $\int \psi(u) du$ and substituting $\psi(u)=s$:
$$
    \int \psi(u) du = u\psi(u) -\int u \psi'(u) du= u\psi(u) - \int \psi^{-1}(s)ds = u\psi(u)-f_\psi(\psi(u)) + C = f^\star_\psi(u) + C
$$
Thus $f^\star_\psi(u) =  \int \psi(u) du - C$. Here $C$ is the same constant of integration picked when defining $f_\psi(z) = \int \psi^{-1}(z) dz + C$. We have the following property (proof in Appendix \ref{app:potentials}):

\begin{restatable}{Lemma}{bregtransform}
\label{lem:breg_transform}
Let $x,y$ be such that $x=\psi(u)$ and $y=\psi(v)$. Then $\breg_{f_\psi}(y\|x) = \breg_{f_\psi^\star}(u\|v)$
\end{restatable}

\section{New local-norm lower-bounds for Bregman divergences}
\label{subsec:BregmanLowerBound}

Let $\psi$ be a potential and $x,y \in \mathbb{R}_+$. We show a general way of obtaining lower-bounds using potential functions, that are of the form:
$$\breg_{f_\psi}(y\|x) \geq \frac{1}{2 w(x)}(x-y)^2$$
Where $w$ is some positive function.

\begin{lemma}
\label{lem:lowerbound}
Let $\psi$ be a potential and $x\in \mathbb{R}_+$ such that $x=\psi(u)$ for some $u$. Let $\phi$ be a non-negative function such that $\psi(u+\phi(u))$ exists. Define the function $m(z) = \frac{\psi(z+\phi(z))-\psi(z)}{\phi(z)}$. For all $0< y \leq \psi(u+\phi(u))$ we have the lower bound: 
$\breg_{f_\psi}(y\|x) \geq \frac{1}{2} \frac{(x-y)^2}{m(\psi^{-1}(x))}$
\end{lemma}
\begin{proof}
Let $v$ be such that $y=\psi(v)$. Using Lemma \ref{lem:breg_transform}, we have $\breg_{f_\psi}(y\|x) = \breg_{f_\psi^\star}(u\|v)$. Using the fact that $f_\psi^\star(u) = \int \psi(u) du - C$, we have:

\begin{align*}
    \breg_{f_\psi^\star}(u\|v) &= f_\psi^\star(u)-f_\psi^\star(v)-{f_\psi^\star}'(v)(u-v) = \int_{v}^u \psi(s) - y(u-v)
\end{align*}

We can visualize $\breg_{f_\psi^\star}(u\|v)$ using the potential function. When $v\leq u$, it is the area with green borders in Figure \ref{fig:case1} and when $u\leq v$, it is the area with green borders in Figure \ref{fig:case2}.

\begin{figure}[h]
    \begin{minipage}{0.5\textwidth}
        \centering
        \includegraphics[width=\textwidth]{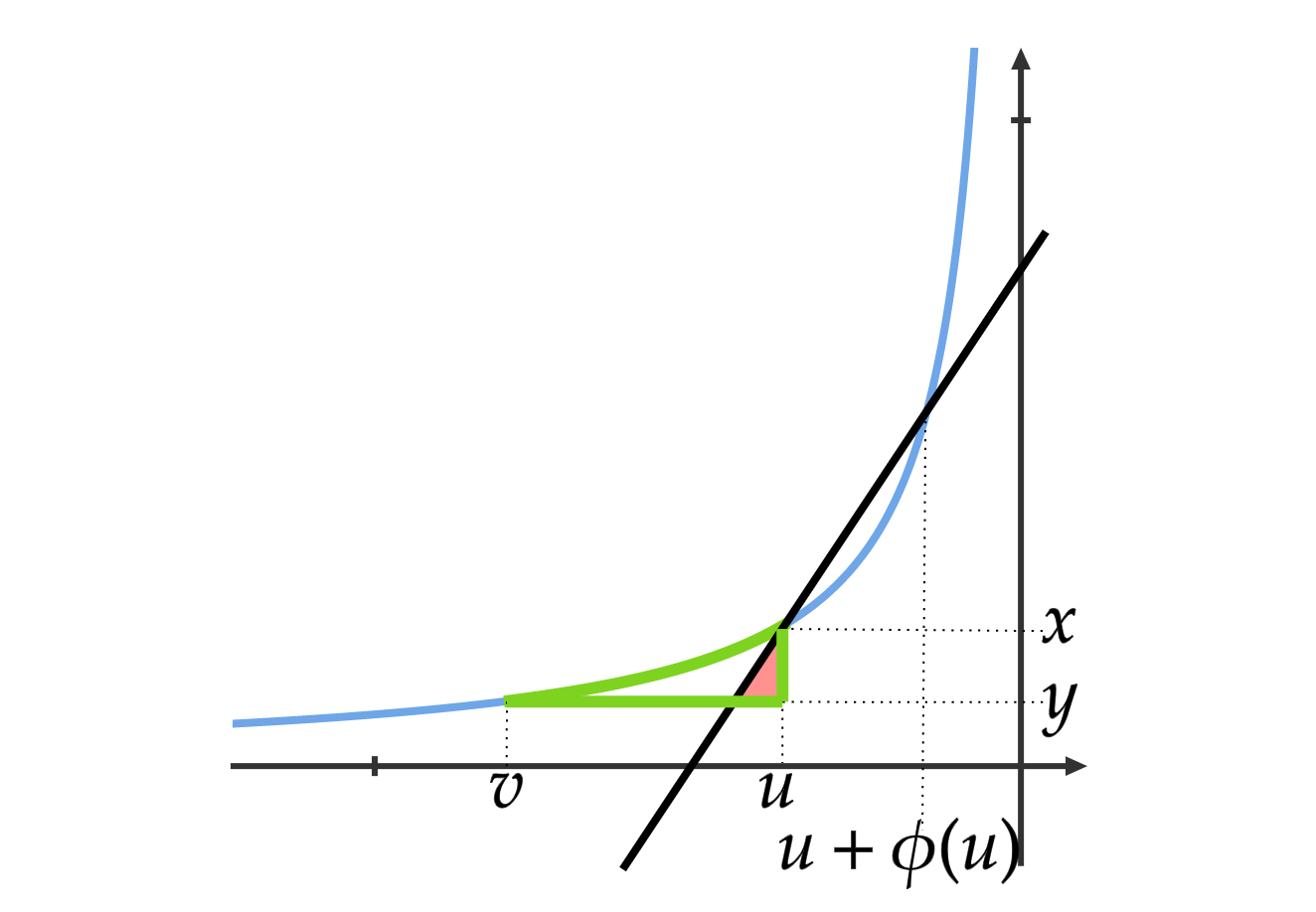} 
        \caption{$v\leq u$}
        \label{fig:case1}
    \end{minipage}%
    \begin{minipage}{0.5\textwidth}
        \centering
        \includegraphics[width=\textwidth]{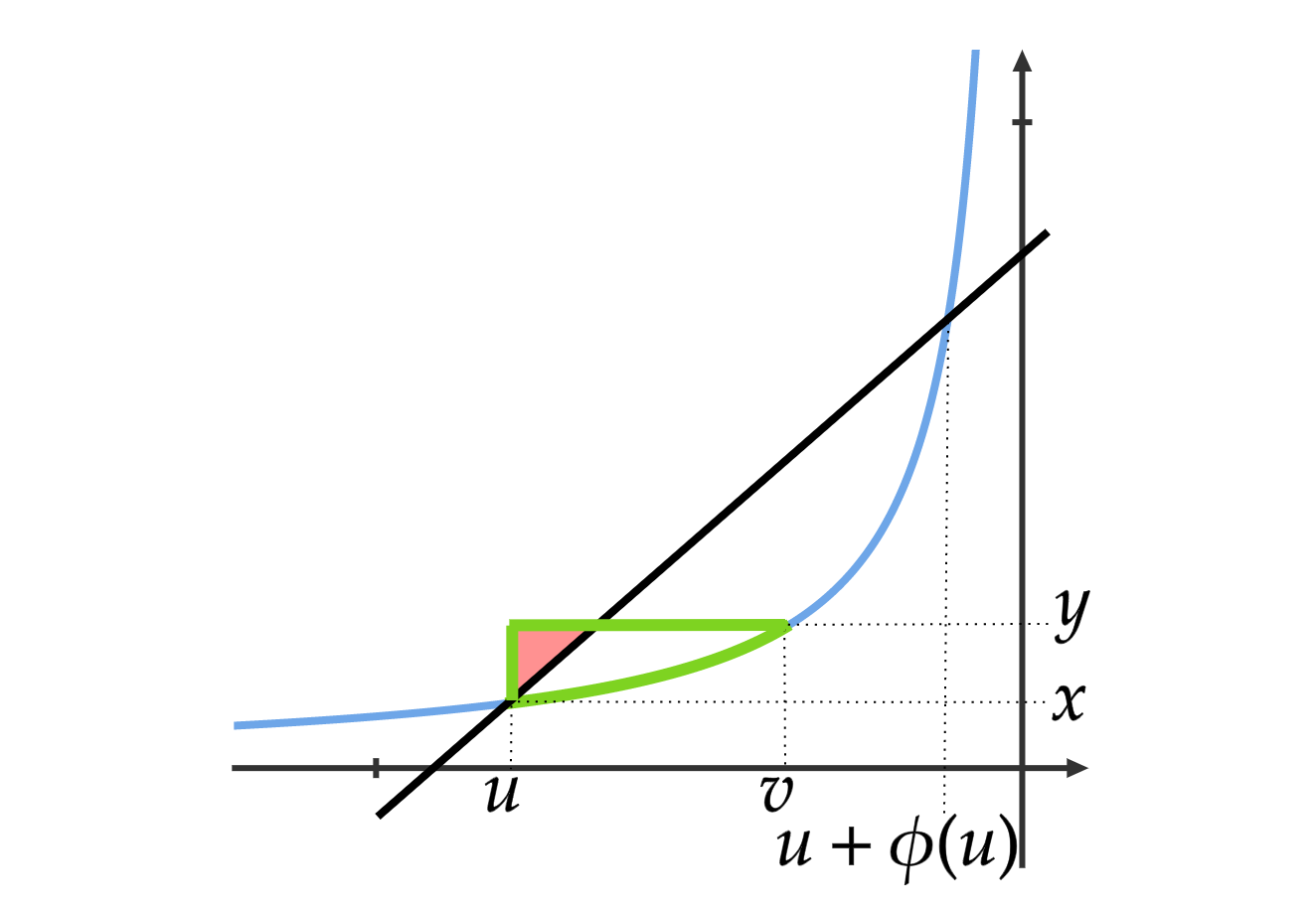} 
        \caption{$u\leq v\leq u+\phi(u)$}
        \label{fig:case2}
    \end{minipage}
\end{figure}

Consider the line passing through $(u,x)$ and $(u+\phi(u),\psi(u+\phi(u))$. Its slope is $m(u) \geq \psi'(u)> 0$. In both cases, the height of the red triangle is $|x-y|$ and its base is $\frac{|x-y|}{m(u)}$. So, the area of the red triangle will be $\frac{1}{2} \frac{(x-y)^2}{m(u)}$. Since the triangle is always smaller than $\breg_{f_\psi^\star}(u\|v)$, we have the lower bound $\breg_{f_\psi}(y\|x) \geq \frac{1}{2} \frac{(x-y)^2}{m(\psi^{-1}(x))}$.
\end{proof}

In the context of online learning, local-norm lower-bounds have been studied before, see for example \cite{DBLP:journals/corr/abs-1912-13213}. However, these relied upon Taylor's theorem to show that $\breg_{f_\psi}(y\|x) = \frac{1}{2}(x-y)^2 f''_\psi(z)$ for some $z \in [x,y]$. Then, they used further conditions on $x,y$ to argue that $c f''_\psi(x) \leq f''_\psi(z)$ for some positive constant $c$ and thus arrive at $\breg_{f_\psi}(y\|x) \geq \frac{c}{2}(x-y)^2 f''_\psi(x)$. We generalize this argument in Lemma \ref{lem:lowerbound}, through which we are able to generate a more rich class of lower-bounds. We illustrate with an example below:

\begin{corollary}
\label{cor:LogBarrierLowerbound}
Let $\psi(u)=-1/u$ in the domain $(-\infty,0)$. For $x,y \in (0,1]$, we have the lower-bound $$\breg_{f_\psi}(y\|x) = \frac{y}{x}-1-\ln\left( \frac{y}{x} \right)\geq \frac{1}{2} \frac{(x-y)^2}{x}$$
\end{corollary}
\begin{proof}
For any $x \in (0,1]$, let $u \in (-\infty,-1]$ be such that $\psi(u)=x$. Let $\phi(u) = -1-u$. Clearly, $\phi(u) \geq 0$ and $\psi(u+\phi(u)) = \psi(-1)=1$. We have $$m(u) = \frac{\psi(u+\phi(u))-\psi(u)}{\phi(u)} = \frac{1+\frac{1}{u}}{-1-u} = \frac{-1}{u} = \psi(u) = x$$
Applying Lemma \ref{lem:lowerbound}, we have the lower-bound for all $0<y\leq 1$:
$$\breg_{f_\psi}(y\|x) = \frac{y}{x}-1-\ln\left( \frac{y}{x} \right) \geq \frac{1}{2} \frac{(x-y)^2}{m(\psi^{-1}(x))} =\frac{1}{2} \frac{(x-y)^2}{x} $$
\end{proof}
\begin{figure}[h]
    \centering
    \includegraphics[width = 0.47\textwidth]{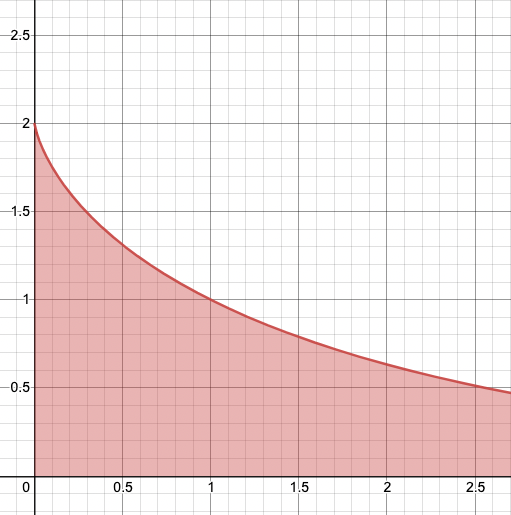}
    \caption{$\frac{y}{x}-1-\ln\left( \frac{y}{x} \right) \geq  \frac{1}{2} \frac{(x-y)^2}{x} $}
    \label{fig:inequality}
\end{figure}
The result of Corollary \ref{cor:LogBarrierLowerbound} is illustrated in Figure \ref{fig:inequality}. The shaded region is  $\lbrace(x,y): x\geq 0, y\geq 0, \frac{y}{x}-1-\ln\left( \frac{y}{x} \right) \geq  \frac{1}{2} \frac{(x-y)^2}{x} \rbrace$. Clearly the region $\{(x,y):0\leq x\leq 1, 0\leq y\leq 1\}$ is within the shaded region.

\section{Full-Information FTRL and AdaFTRL}
\label{subsec:FTRL}

The iterates of FTRL with the regularizer $F_\psi(x) = \sum_{i=1}^n [f_\psi(x(i)) - f_\psi(1/n)]$ for some potential function $\psi$ and positive learning rates $\{\eta_t\}_{t=0}^T$, are of the form:
$$p_{t+1} = \arg \min_{q\in \Delta_n} \left[ F_\psi(q)+\eta_t \sum_{s=1}^t l_s^\top q \right]$$
Since $F_\psi$ is Legendre, the point $p_{t+1}$ always exists strictly inside $\Delta_n$. 
\cite{DBLP:journals/corr/abs-1912-13213} and \cite{DBLP:conf/alt/JoulaniGS17, DBLP:journals/tcs/JoulaniGS20} provide general purpose regret analysis of FTRL. For the sake of completeness, we show a simple way of analyzing FTRL when the  action set is $\Delta_n$ and the regularizer chosen is of the form $F_\psi(x) = \sum_{i=1}^n [f_\psi(x(i)) - f_\psi(1/n)]$ in Appendix \ref{app:ftrl_regret}.

The AdaFTRL strategy picks a specific sequence of learning rate $\eta_t$ based on the history $H_t$. This strategy was analyzed in \cite{DBLP:journals/tcs/OrabonaP18} and a simpler analysis was given by \cite{koolen_2016}. Our analysis is adapted from \citet[Section E.2.1]{hadiji2020adaptation}. We consider the adaptive learning rate: $$\eta_t = \frac{\alpha}{\beta + \sum_{s=1}^t M_s(\eta_{s-1})}$$
Where  $M_t(\eta) = \sup_{q \in \Delta_n} \left[ l_t^\top(p_t-q) -  \frac{1}{\eta} \breg_{F_\psi}(q\|p_t)\right]$, is the \textit{Mixability Gap} and $\alpha,\beta>0$. Since $q = p_t$ is a feasible solution for this optimization problem, we have $M_t(\eta) \geq 0$. Let $p_t^\star$ be the optimal value of $q$ in the optimization. We have the upper bound $$M_t(\eta) = l_t^\top(p_t-p_t^\star) - \frac{1}{\eta}\breg_{F_\psi}(p_t^\star\|p_{t})  \leq l_t^\top(p_t-p_t^\star)  \leq 2 \|l_t\|_\infty$$ Since $M_t(\eta)$ are non-negative and bounded, the sequence $\eta_t$ is non-increasing.

\begin{theorem}
\label{thm:LogBarrierRegret}
If the regularizer is the log-barrier $F_\psi(x) = \sum_{i=1}^n [\log(1/n)-\log(x(i))]$ then for any $i\in [n]$, $\epsilon \in (0,1]$ and any sequence of losses $l_1,\dots,l_T$, the iterates of AdaFTRL satisfy the regret inequality $\sum_{t=1}^T l_t^\top (p_t-\textbf{1}^i_\epsilon) $:
$$\leq n\log(\nicefrac{1}{\epsilon}) \left( \frac{\beta}{\alpha} + \frac{2 \sup_t\|l_t\|_\infty}{\alpha}\right) + 2\sup_t\|l_t\|_\infty +  \sqrt{\sum_{t=1}^T p_t ^\top l_t^2} \left( \frac{n\log(\nicefrac{1}{\epsilon})}{\sqrt{\alpha}}+ \sqrt{\alpha} \right)$$
\end{theorem}
\begin{proof} The log-barrier regularizer $F_\psi(x) = \sum_{i=1}^n [\log(1/n)-\log(x(i))]$ is obtained by using the potential $\psi(u) = -1/u$ on the domain $(-\infty,0)$.  Using Corollary \ref{cor:LogBarrierLowerbound}, we have the lower-bound:
\begin{align*}
   \breg_{F_\psi}(p_t^\star\|p_{t}) = \sum_{i=1}^n \breg_{f_\psi}(p_t^\star(i)\|p_{t}(i)) \geq \sum_{i=1}^n \frac{1}{2} \frac{(p_t(i)-p^\star_t(i))^2}{p_t(i)} 
\end{align*}
This gives us the upper-bound:
\begin{align*}
    M_t(\eta) &= l_t^\top(p_t-p_t^\star) - \frac{1}{\eta}\breg_{F_\psi}(p_t^\star\|p_{t}) \leq \sum_{i=1}^n \left[ l_t(i)(p_t(i) - p_t^\star(i)) - \frac{(p_t(i)-p^\star_t(i))^2}{2\eta p_t(i)} \right]\\
    &\leq \sum_{i=1}^n \sup_{s\in \mathbb{R}}\left[ l_t(i)s - \frac{1}{2\eta} \frac{s^2}{p_t(i)} \right] \leq \frac{\eta}{2} \sum_{n=1}^n  p_t(i) l_t(i)^2 = \frac{\eta}{2} p_t ^\top l_t^2
\end{align*}
Thus, we have $$\frac{M_t(\eta_{t-1})}{\eta_{t-1}} \leq \frac{1}{2} p_t ^\top l_t^2$$
Applying Theorem \ref{thm:AdaFTRL_regret}(Appendix \ref{app:ftrl_regret}), for any $i \in [n]$ and $\epsilon \in (0,1]$ we have that $\sum_{t=1}^T l_t(p_t-\textbf{1}^i_\epsilon)$:
\begin{align*}
    \leq F_\psi(\textbf{1}^{i}_\epsilon) \left( \frac{\beta}{\alpha} + \frac{2 \sup_t\|l_t\|_\infty}{\alpha}\right) + 2\sup_t\|l_t\|_\infty +  \sqrt{\sum_{t=1}^T p_t ^\top l_t^2} \left( \frac{F_\psi(\textbf{1}^{i}_\epsilon)}{\sqrt{\alpha}}+ \sqrt{\alpha} \right) 
\end{align*}
The term  $F_\psi(\textbf{1}^{i}_\epsilon)$ can be bounded as:
\begin{align*}
    F_\psi(\textbf{1}^{i}_\epsilon) &= n\log(1/n) - (n-1)\log(\epsilon/n) - \log((1-\epsilon) + \epsilon/n)\\
    &\leq n \log(1/n) - n\log(\epsilon/n) = n\log(1/\epsilon)
\end{align*}
\end{proof}

\noindent For $p\in \Delta_n$ and regularizer $F_\psi$, the stability term $\Psi$ is defined as $$\Psi_{p}(l) = \sup_{q \in \Delta_n} \left[  l^\top(p-q) -  \breg_{F_\psi}(q\|p)\right]$$
Observe that $\eta M_t(\eta) = \Psi_{p_t}(\eta l_t)$. For the log-barrier regularizer, we have  $M_t(\eta) \leq \eta p_t^\top l_t^2$. Thus,  $\Psi_p(l) \leq p^\top l^2$ for all $l \in \mathbb{R}^n$. Previously, the only known way to achieve $\Psi_p(l) \leq p^\top l^2$ was by using the negative-entropy regularizer along with the assumption $l \geq -\textbf{1}$ (See \citet[Eq. 6 ]{DBLP:journals/corr/abs-1907-05772} or \citet[Eq. 37.15]{lattimore_szepesvari_2020}).

\section{Scale-free bandit regret bounds}
\label{sec:proof}

\main*
\begin{proof} Suppose the best arm in hindsight is $i_\star = \arg\min_{i \in [n]} \sum_{t=1}^T l_t(i) $. Let $\textbf{1}^{i_\star}$ be the vector with $\textbf{1}^{i_\star}(i_\star)=1$ and $\textbf{1}^{i_\star}(i)=0$ for all $i\neq i_\star$. Let $\textbf{1}^{i_\star}_\epsilon = (1-\epsilon)\textbf{1}^{i_\star} + \epsilon/n$.
The exptected regret of Algorithm \ref{alg:SF_MAB} is:
\begin{align*}
    \mathbb{E}\left[ \sum_{t=1}^T l_t(i_t) - l_t(i^\star)  \right] &= \mathbb{E}\left[ \sum_{t=1}^T l_t^\top (p_t' - \textbf{1}^{i_\star})  \right] = \mathbb{E} \left[ \sum_{t=1}^T l_t^\top (\textbf{1}^{i_\star}_\epsilon - \textbf{1}^{i_\star})  \right] + \mathbb{E} \left[ \sum_{t=1}^T l_t^\top (p_t' - \textbf{1}^{i_\star}_\epsilon)  \right] \\
    &=  \underbrace{\mathbb{E} \left[ \sum_{t=1}^T l_t^\top (\textbf{1}^{i_\star}_\epsilon - \textbf{1}^{i_\star})  \right]}_\textrm{(1)} + \underbrace{\mathbb{E} \left[ \sum_{t=1}^T l_t^\top (p_t - \textbf{1}^{i_\star}_\epsilon)  \right]}_\textrm{(2)} + \underbrace{\mathbb{E} \left[ \sum_{t=1}^T l_t^\top (p_t' - p_t)  \right]}_\textrm{(3)} 
\end{align*}
For term (1), we have:
\begin{align*}
    \mathbb{E} \left[ \sum_{t=1}^T l_t^\top (\textbf{1}^{i_\star}_\epsilon - \textbf{1}^{i_\star})  \right] &= \sum_{t=1}^T l_t^\top (\textbf{1}^{i_\star}_\epsilon - \textbf{1}^{i_\star}) \leq 2\epsilon \left\|\sum_{t=1}^T l_t \right\|_\infty = 2\epsilon S_\infty \label{eq:first_term}
\end{align*}
For term (2), we use the fact that $\mathbb{E}[\tilde{l}_t] = l_t$:
\begin{align*}
    \mathbb{E} \left[ \sum_{t=1}^T l_t^\top (p_t - \textbf{1}^{i_\star}_\epsilon)  \right]  &=\mathbb{E} \left[ \sum_{t=1}^T \tilde{l}_t^\top (p_t - \textbf{1}^{i_\star}_\epsilon)  \right]
\end{align*}
Since Algorithm \ref{alg:SF_MAB} runs log-barrier regularized AdaFTRL with the loss sequence $\tilde{l}_1,\dots, \tilde{l}_T$, we can bound the sum inside the expectation using Theorem \ref{thm:LogBarrierRegret} as $\sum_{t=1}^T \tilde{l}_t^\top (p_t - \textbf{1}^{i_\star}_\epsilon) $:
\begin{align*}
   \leq  \log(\nicefrac{1}{\epsilon}) \left(1 + 2 \sup_t\|\tilde{l}_t\|_\infty \right) + 2\sup_t\|\tilde{l}_t\|_\infty +  \sqrt{n \sum_{t=1}^T p_t ^\top \tilde{l}_t^2} \left( \log(\nicefrac{1}{\epsilon})+ 1 \right) \tag{$*$}
\end{align*}
Consider the term $\sup_t\|\tilde{l}_t\|_\infty$:
\begin{align*}
    \sup_t\|\tilde{l}_t\|_\infty &= \sup_t \frac{|l_t(i_t)|}{p'_t(i_t)} = \sup_t \frac{|l_t(i_t)|}{(1-\gamma_{t-1})p_t(i_t) + \gamma_{t-1}/n}\leq n \sup_t \frac{|l_t(i_t)|}{\gamma_{t-1}}
\end{align*}
Since $\gamma_t$ is a positive, non-increasing sequence:
\begin{align*}
    \sup_t\|\tilde{l}_t\|_\infty &\leq  n  \frac{\sup_t |l_t(i_t)|}{\gamma_{T}} \leq \frac{nL_\infty}{\gamma_T}
\end{align*}
Finally, consider the term $ p_t ^\top \tilde{l}_t^2$:
\begin{align*}
    p_t ^\top \tilde{l}_t^2 & = p_t(i_t) \frac{l_t(i_t)^2}{p'_t(i_t)^2} =  p_t(i_t) \frac{l_t(i_t)^2}{((1-\gamma_{t-1})p_t(i_t)+\frac{\gamma_{t-1}}{n})p'_t(i_t)}
    \leq  \frac{l_t(i_t)^2}{(1-\gamma_{t-1})p'_t(i_t)}
\end{align*}
Since $0\leq \gamma_{t-1}\leq 1/2$, we have $1\leq (1-\gamma_{t-1})^{-1} \leq 2$. Thus:
\begin{align*}
    p_t ^\top \tilde{l}_t^2 &\leq   2\frac{l_t(i_t)^2}{p'_t(i_t)}
\end{align*}
Substituting these bounds in the regret inequality $(*)$, we have $\sum_{t=1}^T \tilde{l}_t^\top (p_t - \textbf{1}^{i_\star}_\epsilon) $:
$$\leq \log(\nicefrac{1}{\epsilon}) + \sqrt{2n \sum_{t=1}^T \frac{l_t(i_t)^2}{p'_t(i_t)}} \left( \log(\nicefrac{1}{\epsilon}) + 1\right) + \frac{2n L_\infty}{\gamma_T} \left( \log(\nicefrac{1}{\epsilon}) + 1\right)$$
Applying expectation, we have $\mathbb{E}\left[ \sum_{t=1}^T \tilde{l}_t^\top (p_t - \textbf{1}^{i_\star}_\epsilon)\right]$:
\begin{align*}
    \leq  \log(\nicefrac{1}{\epsilon}) + \mathbb{E} \left[ \sqrt{2n \sum_{t=1}^T \frac{l_t(i_t)^2}{p'_t(i_t)}}\right] \left( \log(\nicefrac{1}{\epsilon}) + 1\right) + 2n L_\infty \left(\log(\nicefrac{1}{\epsilon}) + 1\right) \mathbb{E} \left[ \frac{1}{\gamma_T}\right]
\end{align*}
For the expectation in the second term, we apply Jensen's inequality:
\begin{align*}
    \mathbb{E} \left[ \sqrt{2n \sum_{t=1}^T \frac{l_t(i_t)^2}{p'_t(i_t)}}\right] &\leq  \sqrt{2n \mathbb{E} \sum_{t=1}^T   \left[\frac{l_t(i_t)^2}{p'_t(i_t)}\right]} = \sqrt{2n \sum_{t=1}^T \sum_{i=1}^n l_t(i)^2} = \sqrt{2nL_2}
\end{align*}
Thus term (2) can be bounded as $\mathbb{E}\left[ \sum_{t=1}^T l_t^\top (p_t - \textbf{1}^{i_\star}_\epsilon)\right]$:
\begin{align*}
    &\leq \log(\nicefrac{1}{\epsilon}) + \sqrt{2nL_2} \left( \log(\nicefrac{1}{\epsilon}) + 1\right) + 2n L_\infty \left(\log(\nicefrac{1}{\epsilon}) + 1\right) \mathbb{E} \left[ \frac{1}{\gamma_T}\right]
\end{align*}

\subsection{Non-Adaptive Exploration}
First, we present a simple way to bound term (3):
\begin{align*}
    \mathbb{E} \left[ \sum_{t=1}^T l_t^\top (p_t' - p_t)  \right] &= \mathbb{E} \left[ \sum_{t=1}^T l_t^\top ((1-\gamma_{t-1})p_t + \gamma_{t-1}/n - p_t)  \right] = \mathbb{E} \left[ \sum_{t=1}^T \gamma_{t-1} l_t^\top (1/n - p_t)  \right]\\
    &\leq  \mathbb{E} \left[ 2 \sum_{t=1}^T \gamma_{t-1} \|l_t\|_\infty  \right] \leq 2 L_\infty \mathbb{E} \left[  \sum_{t=1}^T \gamma_{t-1}   \right]
\end{align*}
Combining the upper-bounds for term (1), (2) and (3), we have $\mathbb{E}\left[ \sum_{t=1}^T l_t(i_t) - l_t(i^\star)  \right]$:
\begin{align*}
    \leq 2\epsilon S_\infty + \log(\nicefrac{1}{\epsilon}) + \sqrt{2nL_2} \left( \log(\nicefrac{1}{\epsilon}) + 1\right) + 2n L_\infty \left(\log(\nicefrac{1}{\epsilon}) + 1\right) \mathbb{E} \left[ \frac{1}{\gamma_T}\right] + 2 L_\infty \mathbb{E} \left[  \sum_{t=1}^T \gamma_{t-1}   \right]
\end{align*}
Pick $\epsilon = (1+S_\infty)^{-1}$ and the exploration rate $\gamma_t = \min(1/2,\sqrt{n/t})$. If $T \geq 4n$, the regret of Algorithm \ref{alg:SF_MAB} with non-adaptive exploration is bounded by:
\begin{align*}
    &\leq 2 + \log(1+S_\infty) + \sqrt{2nL_2} (1 + \log(1+S_\infty)) + 2L_\infty \sqrt{nT} (2 + \log(1+S_\infty))\\
    &\leq \left(2 + \log(1+S_\infty)\right) \left(1 + \sqrt{2nL_2} + 2L_\infty\sqrt{nT}\right)\\
    &=\tilde{\mathcal{O}}(  \sqrt{nL_2} + L_\infty\sqrt{nT} )
\end{align*}

\subsection{Adaptive Exploration}
An alternate way to bound term (3) is:
\begin{align*}
    \mathbb{E} \left[ \sum_{t=1}^T l_t^\top (p_t' - p_t)  \right] &= \mathbb{E} \left[ \sum_{t=1}^T \tilde{l}_t^\top (p_t' - p_t)  \right] = \mathbb{E} \left[ \sum_{t=1}^T \gamma_{t-1} \frac{l_t(i_t)}{p_t(i_t)} (1/n - p_t(i_t))  \right]\\
    &\leq  \mathbb{E} \left[  \sum_{t=1}^T \gamma_{t-1}\frac{|l_t(i_t)|}{p'_t(i_t)}  \right]
\end{align*}
Combining the upper-bounds for term (1), (2) and (3), we have $\mathbb{E}\left[ \sum_{t=1}^T l_t(i_t) - l_t(i^\star)  \right] $:
\begin{align*}
   & \leq 2\epsilon S_\infty + \log(\nicefrac{1}{\epsilon}) + \sqrt{2nL_2} \left( \log(\nicefrac{1}{\epsilon}) + 1\right) + \mathbb{E} \left[\frac{2n L_\infty \left(\log(\nicefrac{1}{\epsilon}) + 1\right)}{\gamma_T} +  \sum_{t=1}^T \gamma_{t-1}\frac{|l_t(i_t)|}{p'_t(i_t)}  \right]
\end{align*}
Consider the expression inside the expectation. Let 
$$\Gamma_t(\gamma) = \frac{ \gamma |l_t(i_t)|}{(1-\gamma)p_t(i_t) + \gamma/n} $$
When $0\leq \gamma \leq 1/2$, we have $0 \leq \Gamma_t(\gamma) \leq n|l_t(i_t)| \leq n L_\infty$. Moreover, we have $$\frac{\Gamma_t(\gamma_{t-1})}{\gamma_{t-1}} = \frac{|l_t(i_t)|}{p'_t(i_t)}$$
Pick $$\gamma_{t} = \frac{n}{2n + \sum_{s=1}^t \Gamma_s(\gamma_{s-1})}$$
We satisfy $0\leq \gamma_t\leq 1/2$. Applying Lemma \ref{lem:summation_lemma}, we have:
\begin{align*}
      &\mathbb{E} \left[  \frac{2n L_\infty \left(\log(\nicefrac{1}{\epsilon}) + 1\right)}{\gamma_T} +  \sum_{t=1}^T \gamma_{t-1}\frac{|l_t(i_t)|}{p'_t(i_t)}  \right] = \mathbb{E} \left[\frac{2n L_\infty \left(\log(\nicefrac{1}{\epsilon}) + 1\right)}{\gamma_T} +  \sum_{t=1}^T \Gamma_t(\gamma_{t-1}) \right] \\
      &\leq 2n L_\infty(2 + L_\infty) \left(\log(\nicefrac{1}{\epsilon}) + 1\right)   + nL_\infty +  \left( 2 L_\infty \left(\log(\nicefrac{1}{\epsilon}) + 1\right) +1 \right) \mathbb{E} \left[\sqrt{2n \sum_{t=1}^T \frac{|l_t(i_t)|}{p'_t(i_t)}}\right] 
\end{align*}
For the expectation above, we apply Jensen's inequality:
\begin{align*}
    \mathbb{E} \left[ \sqrt{2n \sum_{t=1}^T \frac{|l_t(i_t)|}{p'_t(i_t)}}\right] &\leq  \sqrt{2n \mathbb{E} \sum_{t=1}^T   \left[\frac{|l_t(i_t)|}{p'_t(i_t)}\right]} = \sqrt{2 n\sum_{t=1}^T \sum_{i=1}^n |l_t(i)|} = \sqrt{2nL_1}
\end{align*}
Pick $\epsilon = (1+S_\infty)^{-1}$. The regret of Algorithm \ref{alg:SF_MAB} with adaptive exploration is bounded by:
\begin{align*}
    & \leq 2 + \log(1+S_\infty) + \sqrt{2nL_2} \left( \log(1+S_\infty) + 1\right) \\
    & \quad +  2n L_\infty(2 + L_\infty) \left(\log(1+S_\infty) + 1\right)   + nL_\infty +  \left( 2 L_\infty \left(\log(1+S_\infty) + 1\right) +1 \right) \sqrt{2nL_1}\\
    &= \tilde{\mathcal{O}}(\sqrt{nL_2} + L_\infty\sqrt{nL_1})
\end{align*}

\end{proof}

\acks{We thank a bunch of people.}

\bibliography{alt2022-sample}

\appendix

\section{Basic results on potentials}

\label{app:potentials}
Consider a function $g: \mathbb{R}^n \times \mathbb{R} \to \mathbb{R}_+$ defined as $ g(\theta,\lambda) = \sum_{i=1}^n \psi(\theta(i) + \lambda) $ for some potential $\psi$.

\begin{restatable}{Lemma}{uniquelambda}
\label{lem:unique_lambda}
For every $\theta \in \mathbb{R}^n$, there exists a unique $\lambda$ such that $g(\theta, \lambda)=1$
\end{restatable}
\begin{proof} For every $\theta \in \mathbb{R}^n$, we have that $\displaystyle \lim_{\lambda \to -\infty} g(\theta,\lambda) = 0$ and $\displaystyle \lim_{\lambda \to a-\min_i(\theta(i))} g(\theta,\lambda) = +\infty$. As $g$ is monotonically increasing and continuous, by the intermediate value theorem, for every $\theta \in \mathbb{R}^n$ there exists a unique $\lambda$ such that $g(\theta, \lambda)=1$.
\end{proof}

Using Lemma \ref{lem:unique_lambda}, we can define a function $\lambda(\theta)$ such that $ g(\theta,\lambda(\theta)) = \sum_{i=1}^n \psi(\theta(i) + \lambda(\theta))=1 $. Since $\psi(\theta(i) + \lambda(\theta)) \geq 0$ and $\sum_{i=1}^n \psi(\theta(i) + \lambda(\theta))=1$, we can see that the vector $\psi(\theta + \lambda(\theta)) \equiv \{\psi(\theta(i) + \lambda(\theta))\}_{i=1}^n \in \Delta_n$ forms a probability distribution.

\bregtransform*

\begin{proof}
Use the fact that $f^\star_\psi(u) = u\psi(u) - f(\psi(u))$.
\begin{align*}
\breg_{f_\psi}(y\|x) &= \breg_{f_\psi}(\psi(v)\|\psi(u)) = f_\psi(\psi(v)) - f_\psi(\psi(u)) - f_\psi'(\psi(u))(\psi(v)-\psi(u))\\
&=v\psi(v) - f^\star_\psi(v) - (u\psi(u) - f^\star_\psi(u)) - u(\psi(v)-\psi(u))\\
&= f^\star_\psi(u) - f^\star_\psi(v) - {f^\star}_\psi'(v)(u-v) = \breg_{f^\star_\psi}(u\|v)
\end{align*}
\end{proof}

\section{A useful summation}

\begin{restatable}{Lemma}{summationLemma}
\label{lem:summation_lemma}
Let $A>0$ and $0\leq M_t(a) \leq L$ for all $t=1,\dots T$ and $a \in \mathcal{A} \subseteq(0,\infty)$. Consider the expression $$\frac{A}{a_T} + \sum_{t=1}^T M_t(a_{t-1})$$ Where  $$a_t = \frac{\alpha}{\beta + \sum_{s=1}^t M_s(a_{s-1})}$$
Constants $\alpha,\beta > 0$ are chosen such that $a_t \in \mathcal{A}$. If $\frac{M_t(a_{t-1})}{a_{t-1}}\leq g_t$, then we have the upper bound:
$$\frac{A}{a_T} + \sum_{t=1}^T M_t(a_{t-1}) \leq A  \left(\frac{\beta}{\alpha} + \frac{L}{\alpha}\right) + L +  \sqrt{2 \sum_{t=1}^T g_t} \left( \frac{A}{\sqrt{\alpha}} + \sqrt{\alpha} \right) $$
\end{restatable}

\begin{proof}
Substituting for $a_T$ in the above expression, we have:
$$\frac{A}{a_T} + \sum_{t=1}^T M_t(a_{t-1}) = \frac{A \beta}{\alpha} + \left( \frac{A}{\alpha} + 1 \right) \sum_{t=1}^T M_t(a_{t-1})$$

Consider $\left(\sum_{t=1}^T M_t(a_{t-1})\right)^2$

\begin{align*}
    \left(\sum_{t=1}^T M_t(a_{t-1})\right)^2 &= \sum_{t=1}^T  M_t(a_{t-1})^2 + 2\sum_{t=1}^T M_t(a_{t-1})  \sum_{s=1}^{t-1} M_{s}(a_{s-1})\\
    &= \sum_{t=1}^T  M_t(a_{t-1})^2 + 2\sum_{t=1}^T M_t(a_{t-1})  \left(\frac{\alpha}{a_{t-1}} -\beta \right)\\
    &\leq  \sum_{t=1}^T  M_t(a_{t-1})^2 + 2\alpha \sum_{t=1}^T   \frac{M_t(a_{t-1})}{a_{t-1}}\\
    &\leq L \sum_{t=1}^T  M_t(a_{t-1}) + 2\alpha \sum_{t=1}^T   g_t
\end{align*}

Using the fact that $x^2\leq a+bx$ implies that $x \leq \sqrt{a}+b$ for all $a,b,x \geq 0$, we have:

$$\sum_{t=1}^T M_t(a_{t-1}) \leq \sqrt{2\alpha \sum_{t=1}^T g_t} + L$$

Thus, we get:
\begin{align*}
    \frac{A}{a_T} + \sum_{t=1}^T M_t(a_{t-1}) &= \frac{A\beta}{\alpha} + \left( \frac{A}{\alpha}+1\right)  \sum_{t=1}^T M_t(a_{t-1})\leq\frac{A\beta}{\alpha} +\left( \frac{A}{\alpha}+1\right)  \left( \sqrt{2\alpha \sum_{t=1}^T g_t} + L \right) \\
    &=A  \left(\frac{\beta}{\alpha} + \frac{L}{\alpha}\right) + L +  \sqrt{2 \sum_{t=1}^T g_t} \left( \frac{A}{\sqrt{\alpha}} + \sqrt{\alpha} \right)
\end{align*}

\end{proof}

\section{FTRL and AdaFTRL regret bound}
\label{app:ftrl_regret}
Recall the FTRL update:
$$p_{t+1} = \arg \min_{q\in \Delta_n} \left[ F_\psi(q)+\eta_t \sum_{s=1}^t l_s^\top q \right]$$
The iterate $p_{t+1}$ can be expressed in a simple closed form using $\psi$. Let $\theta_t = -\eta_t \sum_{s=1}^t l_s$. The Lagrangian of the above optimization problem is $L(q,\alpha) = F_\psi(q)-\theta_t^\top q - \alpha(1-\textbf{1}^\top q)$, where $\textbf{1}$ is the all ones vector. Taking its derivative with respect to $q(i)$ and equating to $0$, we get:
$$\psi^{-1}(q(i)) =  \theta_t(i) + \alpha \implies q(i) = \psi(\theta_t(i) + \alpha)$$ To compute $\alpha$, we use the fact that $\sum_{i=1}^n q(i)=1$ along with Lemma~\ref{lem:unique_lambda} to show that $\alpha = \lambda(\theta_t)$. Thus, $p_{t+1}$ can be written as:
$$p_{t+1} = \psi(\theta_t + \lambda(\theta_t)) \quad \text{ where } \quad \theta_t = -\eta_t \sum_{s=1}^t l_s$$

\noindent We introduce the  Mixed Bregman in order to simplifies our analysis of FTRL.

\begin{definition}[Mixed Bregman]
For $\alpha,\beta>0$ the $(\alpha,\beta)$-Mixed Bregman of function $F$ is:

$$\breg^{\alpha,\beta}_F(x\|y) = \frac{F(x)}{\alpha}-\frac{F(y)}{\beta} - \frac{\nabla F(y)}{\beta}^\top (x-y).$$
\end{definition}

\noindent The Mixed Bregman is not a divergence as $\breg^{\alpha,\beta}_F(x\|x)$ may not be zero. However, we do have the relation $\alpha \breg^{\alpha,\alpha}_F(x\|y) = \breg_F(x\|y)$.

\begin{restatable}{theorem}{oneFTRLregret}
\label{thm:1FTRL_regret}
For any $p \in \Delta_n$ and any sequence of losses $l_1,\dots,l_T$, the iterates of FTRL satisfy the regret equality $\sum_{t=1}^T l_t^\top (p_t-p) $
$$  = \frac{1}{\eta_T}\left[ \breg_{F_\psi}(p\|p_1)-\breg_{F_\psi}(p\|p_{T+1}) \right] + \sum_{t=1}^T \left[ l_t^\top(p_t-p_{t+1}) - \breg^{\eta_t,\eta_{t-1}}_{F_\psi}(p_{t+1}\|p_{t}) \right] $$
Further, if the sequence $\{\eta_t\}_{t=0}^T$ is non-decreasing, we have the regret inequality $\sum_{t=1}^T l_t^\top (p_t-p) $: 
$$  \leq  \frac{F_\psi(p)}{\eta_T} + \sum_{t=1}^T \left[ l_t^\top(p_t-p_{t+1}) - \frac{1}{\eta_{t-1}}\breg_{F_\psi}(p_{t+1}\|p_{t}) \right] $$
\end{restatable}

\begin{proof}
Note that $\nabla F_\psi(p_{t+1}) = \psi^{-1}(p_{t+1}) = \theta_t+\lambda(\theta_t)$. We also have that $l_t = \frac{\theta_{t-1}}{\eta_{t-1}}-\frac{\theta_t}{\eta_t}$. For any $p\in \Delta_n$, we have $l_t^\top(p_t-p) $:
\begin{align*}
&= l_t^\top(p_{t+1}-p) + l_t^\top(p_t-p_{t+1})= \left(\frac{\theta_{t-1}}{\eta_{t-1}}-\frac{\theta_t}{\eta_t}\right)^\top(p_{t+1}-p) + l_t^\top(p_t-p_{t+1})\\
&= \left(\frac{\nabla F_\psi(p_t)-\lambda(\theta_{t-1})}{\eta_{t-1}} - \frac{\nabla F_\psi(p_{t+1})-\lambda(\theta_t)}{\eta_t}\right)^\top(p_{t+1}-p) + l_t^\top(p_t-p_{t+1})\\
&= \left(\frac{\nabla F_\psi(p_t)}{\eta_{t-1}} - \frac{\nabla F_\psi(p_{t+1})}{\eta_t}\right)^\top(p_{t+1}-p) + \left(\frac{\lambda(\theta_t)}{\eta_t} - \frac{\lambda(\theta_{t-1})}{\eta_{t-1}}\right)^\top(p_{t+1}-p) + l_t^\top(p_t-p_{t+1})\\
&= \left(\frac{\nabla F_\psi(p_t)}{\eta_{t-1}} - \frac{\nabla F_\psi(p_{t+1})}{\eta_t}\right)^\top(p_{t+1}-p)  + l_t^\top(p_t-p_{t+1})
\end{align*}
Note that $\frac{\lambda(\theta_t)}{\eta_t} - \frac{\lambda(\theta_{t-1})}{\eta_{t-1}}$ is a constant vector. So, $\left(\frac{\lambda(\theta_t)}{\eta_t} - \frac{\lambda(\theta_{t-1})}{\eta_{t-1}}\right)^\top(p_{t+1}-p) = 0$.
Let $\alpha$ be any number. Observe that:
$$ \left(\frac{\nabla F_\psi(p_t)}{\eta_{t-1}} - \frac{\nabla F_\psi(p_{t+1})}{\eta_t}\right)^\top(p_{t+1}-p)  = \breg^{\alpha,\eta_{t-1}}_{F_\psi}(p\|p_t) - \breg^{\alpha,\eta_{t}}_{F_\psi}(p\|p_{t+1})- \breg^{\eta_t,\eta_{t-1}}_{F_\psi}(p_{t+1}\|p_t)$$
Taking summation over $t$, we have $\sum_{t=1}^T l_t^\top(p_t-p)$:
\begin{align*}
 &= \sum_{t=1}^T \left[\breg_{F_\psi}^{\alpha,\eta_{t-1}}(p\|p_t) - \breg_{F_\psi}^{\alpha,\eta_{t}}(p\|p_{t+1})\right] + \sum_{t=1}^T \left[ l_t^\top(p_t-p_{t+1}) -  \breg_{F_\psi}^{\eta_t,\eta_{t-1}}(p_{t+1}\|p_t)\right]\\
&=\breg_{F_\psi}^{\alpha,\eta_{0}}(p\|p_1) - \breg_{F_\psi}^{\alpha,\eta_{T}}(p\|p_{T+1}) + \sum_{t=1}^T \left[ l_t^\top(p_t-p_{t+1}) -  \breg_{F_\psi}^{\eta_t,\eta_{t-1}}(p_{t+1}\|p_t)\right]
\end{align*}
Since $p_1=(1/n,\dots,1/n)$, we have $F_\psi(p_1)=0$ and $\nabla F_\psi(p_1)$ is a constant vector. We see that $\nabla F_\psi(p_1)^\top(p-p_1)=0$, so the first term is:
\begin{align*}
    \breg_{F_\psi}^{\alpha,\eta_{0}}(p\|p_1) - \breg_{F_\psi}^{\alpha,\eta_{T}}(p\|p_{T+1}) &= \frac{F_\psi(p_{T+1})}{\eta_T} + \frac{\nabla F_\psi(p_{T+1})^\top(p-p_{T+1})}{\eta_T}\\
    & = \frac{1}{\eta_T} \left[ \breg_{F_\psi}(p\|p_1)-\breg_{F_\psi}(p\|p_{T+1}) \right]
\end{align*}
This completes the proof of the first part.\\

\noindent As $F_\psi(p_{t+1})\geq 0$ and $\eta_t$ are non-increasing we have:
\begin{align*}
    &\breg^{\eta_t,\eta_{t-1}}_{F_\psi}(p_{t+1}\|p_{t}) = \frac{F_\psi(p_{t+1})}{\eta_t}-\frac{F_\psi(p_t)}{\eta_{t-1}} - \frac{\nabla F_\psi(p_t)}{\eta_{t-1}}^\top (p_{t+1}-p_t)\\
    &\geq \frac{F_\psi(p_{t+1})}{\eta_{t-1}}-\frac{F_\psi(p_t)}{\eta_{t-1}} - \frac{\nabla F_\psi(p_t)}{\eta_{t-1}}^\top (p_{t+1}-p_t) = \frac{1}{\eta_{t-1}} \breg_{F_\psi}(p_{t+1}\|p_t)
\end{align*}
Thus, we have $\sum_{t=1}^T l_t^\top(p_t-p)$:
\begin{align*}
    &=\frac{1}{\eta_T}\left[ \breg_{F_\psi}(p\|p_1)-\breg_{F_\psi}(p\|p_{T+1}) \right]  + \sum_{t=1}^T \left[ l_t^\top(p_t-p_{t+1}) -  \breg_{F_\psi}^{\eta_t,\eta_{t-1}}(p_{t+1}\|p_t)\right]\\
    &\leq \frac{1}{\eta_T} \breg_{F_\psi}(p\|p_1) + \sum_{t=1}^T \left[ l_t^\top(p_t-p_{t+1}) -  \breg_{F_\psi}^{\eta_t,\eta_{t-1}}(p_{t+1}\|p_t)\right]\\
    &\leq \frac{F_\psi(p)}{\eta_T}   + \sum_{t=1}^T \left[ l_t^\top(p_t-p_{t+1}) -  \frac{1}{\eta_{t-1}} \breg_{F_\psi}(p_{t+1}\|p_t)\right]
\end{align*}
This completes the proof.
\end{proof}

\noindent Recall that the AdaFTRL strategy picks learning rate: $$\eta_t = \frac{\alpha}{\beta + \sum_{s=1}^t M_s(\eta_{s-1})}$$
Where  $$M_t(\eta) = \sup_{q \in \Delta_n} \left[ l_t^\top(p_t-q) -  \frac{1}{\eta} \breg_{F_\psi}(q\|p_t)\right]$$

\begin{restatable}{theorem}{AdaFTRLregret}
\label{thm:AdaFTRL_regret}
If $M_t(\eta_{t-1})/\eta_{t-1} \leq g_t$, then for any $p \in \Delta_n$ and any sequence of losses $l_1,\dots,l_T$, the iterates of AdaFTRL satisfy the regret inequality $\sum_{t=1}^T l_t^\top (p_t-p) $

$$  \leq F_\psi(p) \left( \frac{\beta}{\alpha} + \frac{ 2 \sup_{t}\|l_t\|_\infty}{\alpha}\right) +  2 \sup_{t}\|l_t\|_\infty +  \sqrt{2\sum_{t=1}^T g_t} \left( \frac{F_\psi(p)}{\sqrt{\alpha}}+ \sqrt{\alpha} \right)  $$
\end{restatable}

\begin{proof}
When using non-increasing $\eta_t$, the regret of FTRL is bounded by Theorem \ref{thm:1FTRL_regret}:
\begin{align*}
    \sum_{t=1}^T l_t^\top(p_t-p) &\leq \frac{F_\psi(p)}{\eta_T}   + \sum_{t=1}^T \left[ l_t^\top(p_t-p_{t+1}) -  \frac{1}{\eta_{t-1}} \breg_{F_\psi}(p_{t+1}\|p_t)\right]\\
    &\leq  \frac{F_\psi(p)}{\eta_T}   + \sum_{t=1}^T M_t(\eta_{t-1})
\end{align*}
Using the fact that $0\leq M_t(\eta)\leq 2\sup_{t}\|l_t\|_\infty$ and applying Lemma \ref{lem:summation_lemma}, we have $\sum_{t=1}^T l_t^\top (p_t-p) $
$$  \leq F_\psi(p) \left( \frac{\beta}{\alpha} + \frac{ 2 \sup_{t}\|l_t\|_\infty}{\alpha}\right) +  2 \sup_{t}\|l_t\|_\infty +  \sqrt{2\sum_{t=1}^T g_t} \left( \frac{F_\psi(p)}{\sqrt{\alpha}}+ \sqrt{\alpha} \right)  $$
\end{proof}

\end{document}